%% file: paper.tex
\documentclass[reqno,11pt]{article} 
\usepackage[utf8]{inputenc}

\usepackage{xxx}

\renewcommand{\para}[1]{\paragraph{#1}}

\usepackage[tt=false]{libertine}

\usepackage{graphicx}
\usepackage{geometry}

\usepackage{amsmath}

\usepackage{amssymb}
\usepackage{amsfonts}
\usepackage{xcolor}
\definecolor{darkblue}{HTML}{000080}

\usepackage{url}
\usepackage[colorlinks=true, allcolors=blue, urlcolor=black]{hyperref}

\usepackage{dsfont}

\usepackage{mathrsfs}

\usepackage{mathtools}
\mathtoolsset{showonlyrefs}

\usepackage{nicefrac}
\usepackage[shortlabels]{enumitem}

\usepackage{natbib}
\setcitestyle{authoryear} 
\setcitestyle{square}
\setcitestyle{semicolon} 

\usepackage[labelformat=simple]{subcaption}

\usepackage{algorithm}
\usepackage{algorithmic}

\usepackage{amsthm}
\newtheorem{theorem}{Theorem}
\newtheorem{lemma}[theorem]{Lemma}
\newtheorem{proposition}[theorem]{Proposition}
\newtheorem{corollary}[theorem]{Corollary}

\makeatletter
\newtheorem*{rep@theorem}{\rep@title}
\newcommand{\newreptheorem}[2]{%
\newenvironment{rep#1}[1]{%
 \def\rep@title{#2 \ref{##1}}%
 \begin{rep@theorem}}%
 {\end{rep@theorem}}}
\makeatother

\newreptheorem{theorem}{Theorem}
\newreptheorem{lemma}{Lemma}

\usepackage[finalold]{trackchanges}

\addeditor{Place_Holder}
\addeditor{TX}
\addeditor{NJ}

\title{Q* Approximation Schemes for Batch Reinforcement Learning: \\A Theoretical Comparison}

\usepackage{authblk}
\author[ ]{Tengyang Xie}
\author[ ]{ Nan Jiang}
\affil[ ]{Department of Computer Science, University of Illinois at Urbana-Champaign}
\affil[ ]{\{\href{mailto:tx10@illinois.edu}{\texttt{tx10}}, \href{mailto:nanjiang@illinois.edu}{\texttt{nanjiang}}\}\texttt{@illinois.edu}}

\date{}

\begin{document}
	
\maketitle



\input{sections/abstract.tex}
\input{sections/main_arxiv.tex}

\bibliographystyle{plainnat}
\bibliography{paper}



\clearpage

\appendix
\onecolumn

\begin{center}
{\LARGE Appendix}
\end{center}


\input{sections/proof_arxiv.tex}
\input{sections/appendix_arxiv.tex}

\end{document}

%% file: sections/abstract.tex
\begin{abstract}
We prove performance guarantees of two algorithms for approximating $Q^\star$ in batch reinforcement learning. Compared to classical iterative methods such as Fitted Q-Iteration---whose performance loss incurs quadratic dependence on horizon---these methods estimate (some forms of) the Bellman error and enjoy linear-in-horizon error propagation, a property established for the first time for algorithms that rely solely on batch data and output stationary policies. 
One of the algorithms uses a novel and explicit importance-weighting correction to overcome the infamous ``double sampling'' difficulty in Bellman error estimation, and does not use any squared losses. Our analyses reveal its distinct characteristics and potential advantages compared to classical algorithms. 
\end{abstract}

%% file: sections/main_arxiv.tex

\section{Introduction}
We study value-function approximation for batch-mode reinforcement learning (RL), which is central to the success of modern RL as many popular off-policy deep RL algorithms find their prototypes in this literature. These algorithms are typically \emph{iterative}, that is, they solve a series of optimization problems, aiming to mimic each step of value- or policy-iteration~\citep{puterman2014markov}. 

In the setting of general function approximation, however, not only the iterative style causes instability in practice, but it also brings several theoretical issues, which have been made abundantly clear in existing analyses~\citep[e.g.,][]{munos2003error, munos2007performance, antos2008learning, farahmand2010error, chen2019information}: 

\para{(A) Quadratic Dependence on Horizon} The performance loss of most iterative methods incur quadratic dependence on the effective horizon, i.e., $\Ocal(\frac{1}{(1-\gamma)^2})$, and this is tight for the popular Approximate Value/Policy Iteration (AVI/API)~\citep{scherrer2012use}. One typical way this occurs in AVI analyses is through the use of (some fine-grained variants of) the following result from~\citet{singh1994upper}, that the performance loss of a policy greedy w.r.t.~some $Q$ is bounded by
\begin{align} \label{eq:singhyee}
\frac{2 \|Q - Q^\star\|_\infty}{1-\gamma},
\end{align}
and translating $\|Q - Q^\star\|$ to the quantities that the algorithm actually optimizes incurs at least another factor of horizon. Such a quadratic dependence is significantly worse than the ideal linear dependence, the best one could hope for~\citep{scherrer2014approximate}. 

While linear-in-horizon algorithms exist, they often require interactive access to the environment (to collect new data using policies of the algorithm's choice), or the knowledge of transition probabilities to compute the true expectation in the Bellman operators,\footnote{At the minimum, two i.i.d.~next-states must be drawn from the same state-action pair, known as the \emph{double sampling} trick~\citep{baird1995residual}, which is unrealistic in non-simulator problems.} and few of them apply to the batch learning setting.\footnote{Exceptions exist when we are allowed to output complex non-stationary policies; see Section~\ref{sec:related} for details.} \emph{Are there batch algorithms for $Q^\star$ that incur linear-in-horizon dependence?}

\para{(B) Characterization of Distribution Shift} One of the central challenges in RL is the \emph{distribution shift}, that the computed policy may induce a state (and action) distribution different from what it is trained on. Existing analyses characterize this effect using the \emph{concentrability coefficients}~\citep{munos2007performance}, with a typical definition being the density ratio (or \emph{importance weights}) between the state distribution induced at \emph{a particular time step} by some non-stationary policy and the data distribution. These ``per-step'' definitions can be very loose even in the uncontrolled setting (Section~\ref{sec:con_coeff}) and sometimes very complicated~\citep{farahmand2010error}. \emph{Are there algorithms whose distribution shift effects are characterized by elegantly and tightly defined quantities?}

\para{(C) Function Approximation Assumptions} Existing analyses require strong expressivity assumptions on the function classes, such as approximate closedness under Bellman update~\citep[see inherent Bellman errors;][]{munos2008finite}. \emph{Are there algorithms with provable guarantees under somewhat weaker conditions?}

\para{(D) Squared-to-Average Conversion} Most batch RL algorithms heavily rely on the squared loss, but bounding the performance loss (which we eventually care about) with squared-loss objectives (which we optimize) often goes through multiple relaxations, including adding point-wise absolute values and communicating between $\ell_1$ and $\ell_2$ norms with Jensen's inequality, reflecting a significant gap between the actual objective (maximizing return) and the surrogate squared loss. On the other hand, we know such indirectness is not necessary in RL from the policy-gradient type algorithms~\citep{sutton2000policy, williams1992simple, kakade2002approximately}, but they cannot be applied in the batch setting due to on-policy roll-outs. \emph{Are there batch algorithms whose loss functions are more directly connected to the expected return?}

\vspace{1em}
In this paper we present novel analyses of two algorithms, \algsq (which has been analyzed by~\citet{chen2019information}) and \algavg (which is novel), and provide positive answers to all questions above. A simple telescoping argument (Section~\ref{sec:telepd}) shows that both algorithms enjoy linear-in-horizon error propagation---which immediately improves the previous bound of~\citet{chen2019information} for \algsq---and the distribution shift effects can be characterized by simple notions of concentrability coefficients that are significantly tighter than previous per-step definitions, which address \qhor and \qcon. By carefully examining the difference between the two algorithms, we further show that \algavg, a novel algorithm that uses explicit importance-weighting correction and plain average objectives (without squared loss) does not suffer from the looseness of squared-to-average conversion, and comes with automatically augmented expressivity for its importance-weight class, addressing \qfun and \qsq. 


\input{sections/table1}

\section{Preliminaries}
\label{sec:bg}

\subsection{Markov Decision Processes (MDPs)}

An (infinite-horizon discounted) MDP~\citep{puterman2014markov} is a tuple $(\mathcal{S}$, $\mathcal{A}$, $P$, $R$, $\gamma, d_0)$: $\Scal$ and $\Acal$ are the finite state and the finite action spaces, respectively, whose cardinalities can be arbitrarily large. $P:\mathcal S \times \mathcal A \to \Delta (\mathcal S)$ is the transition function (we use $\Delta(\cdot)$ to denote the probability simplex), $R: \mathcal{S} \times \mathcal{A} \rightarrow [0, R_{\max}]$ is the reward function, and $\gamma \in [0,1)$ is a parameter that characterizes how rewards are discounted over time. $d_0 \in \Delta(\Scal)$ is the initial state distribution. 

A (stochastic) policy, $\pi: \mathcal{S} \to \Delta(\mathcal A)$, induces a random trajectory $s_0,a_0,r_0, s_1, a_1, r_1, \ldots$ with the following generative process: $s_0 \sim d_0$, $a_t \sim \pi(\cdot | s_t)$, $r_t = R(s_t, a_t)$, $s_{t+1} \sim P(\cdot | s_t, a_t)$, $\forall t \ge 0$. The ultimate goodness of a policy is measured by the expected discounted return (w.r.t.~the initial state distribution), defined as $J(\pi) \coloneqq \E[\sum_{t = 0}^{\infty} \gamma^t r_t | s_0 \sim d_0, \pi]$. There always exists a policy $\pi^\star$ that maximizes the expected return for any initial state distribution. 

It will be useful to define the (state-)value function $V^{\pi}(s) \coloneqq \E[\sum_{t = 0}^{\infty} \gamma^t r_t | s_0 = s, \pi]$ and the Q-function $Q^{\pi}(s,a) \coloneqq \E[\sum_{t = 0}^{\infty} \gamma^t r_t | s_0 = s, a_0 = a, \pi]$. Let $V^\star$ and $Q^\star$ be the shorthand for $V^{\pi^\star}$ and $Q^{\pi^\star}$. All value functions are bounded in $[0,V_{\max}]$, where $V_{\max} \coloneqq R_{\max}/(1 - \gamma)$. It is also known that the greedy policy of $Q^\star$, defined as $\pi_{Q^\star}(s) = \argmax_{a\in\Acal} Q^\star(s,a)$,\footnote{With a slight abuse of notations we treat deterministic policies---which are stochastic policies that put all probability mass on a single action for each state---as of type $\Scal\to\Acal$.} is an optimal policy $\pi^\star$.

Define the Bellman optimality operator: $(\Tcal Q)(s,a) \coloneqq R(s,a) + \gamma \E_{s' \sim P(\cdot|s,a)}[\max_{a' \in \mathcal A}Q(s',a')]$ for any $Q \in \RR^{\Scal\times\Acal}$. $Q^\star$ is the unique fixed point of $\mathcal T$, that is, $\Tcal Q^\star = Q^\star$. We also use $Q(s,\pi)$ as the shorthand for $\sum_{a \in \mathcal A} \pi(a|s)Q(s,a)$.

Another concept crucial to this paper is the normalized discounted state occupancy: 
\begin{align}
d_{\pi}(s) \coloneqq (1 - \gamma) \sum_{t = 0}^\infty \gamma^t \Pr\left[s_t = s \middle| s_0 \sim d_0, \pi\right].
\end{align}
The state-action occupancy $d_{\pi}(s,a)$ is defined similarly and satisfies $d_{\pi}(s,a) = d_{\pi}(s)\pi(a|s)$.

\subsection{Batch Value-Function Approximation} \label{sec:vfa}
\paragraph{Setup} We are concerned with approximating $Q^\star$ in the batch RL setting, where a dataset $\Dcal$ consisting of $n$ $(s,a,r,s')$ tuples is given, and we cannot interact with the MDP to obtain new data. We adopt the following data generation protocol from~\citet{chen2019information}, that the tuples are i.i.d.\footnote{In reality, the transition tuples extracted from the same trajectory are in general dependent, which can be handled by concentration inequalities for dependent processes with mixing assumptions~\citep[see e.g.,][]{antos2008learning}.}~as $(s,a) \sim \mu$, $r = R(s,a)$, $s' \sim P(\cdot|s,a)$, and  $\mu$ is fully supported on $\Scal\times\Acal$.

\paragraph{Function Approximation} We assume access to a function class $\Qcal \subset [0, \Vmax]^{\Scal\times\Acal}$, and focus on algorithms that approximate $Q^\star$ with some $Q \in \Qcal$ and output its greedy policy $\pi_{Q}$. This automatically implies a policy class $\Pi_{\Qcal} \coloneqq \{\pi_Q: Q\in\Qcal\}$, from which the output policy will be chosen. Some algorithms require additional function classes, which we introduce later. We assume all function classes have finite cardinalities  for simplicity when analyzing statistical errors, as they are not our main focus and extension to continuous classes with e.g., finite VC-type dimensions~\citep{natarajan1989learning} are standard.

A representative algorithm for this setting is Fitted Q-Iteration (FQI)~\citep{ernst2005tree}, which can be viewed as the theoretical prototype of the popular DQN algorithm~\citep{mnih2015human}: After initializing $Q_0 \in \Qcal$ arbitrarily, we iteratively compute $Q_t$  as
\begin{align}
\label{eq:fqi_update_rule}
Q_t = \argmin_{Q \in \Qcal} \ell_{\mathcal D}(Q;Q_{t - 1}),
\end{align}
where 
\begin{align}
\label{eq:fqiobj}
\ell_{\mathcal D}(Q;Q') \coloneqq &~ \frac{1}{n} \sum_{(s,a,r,s') \in \mathcal D} \left(Q(s,a) - r - \gamma \max_{a' \in \mathcal A} Q'(s',a') \right)^2.
\end{align}
We will discuss the relationship between FQI (and iterative methods in general) and algorithms we analyze. 

\paragraph{Marginalized Importance Weights} 
We define the importance weight of any policy $\pi$ to be the ratio between its normalized discounted state-action occupancy and the data distribution: 
$$\wpi{\pi}(s,a) \coloneqq \frac{d_{\pi}(s,a)}{\mu(s,a)}.$$
Such functions are of vital importance to us, as in Section~\ref{sec:mabo} we model them with function approximation to explicitly correct distribution mismatch. Their norms also characterize the exploratoriness of the data distribution, which are closely related to the \emph{concentrability coefficients} in prior analyses~\citep{munos2007performance, antos2008learning, farahmand2010error, chen2019information}. 

\paragraph{Additional Notations}
We use the shorthand $\E_{\mu}[\cdot]$ for the population expectation of function of $(s,a,r,s')$ drawn from the data distribution, and $\E_{\Dcal}[\cdot]$ for its sample-based approximation. When the function only depends on $(s,a)$, we further omit the function arguments for brevity; for example, $\E_{\mu}[Q^2] \coloneqq \E_{(s,a)\sim \mu}[Q(s,a)^2]$. 
It will also be convenient to define the $\mu$-weighted $2$-norm  $\|\cdot\|_{2, \mu}^2 \coloneqq \E_{\mu}[(\cdot)^2]$.


\section{Related Work} \label{sec:related}
\para{Linear-in-horizon Analyses} As mentioned in the introduction, most of the existing linear-in-horizon results do not apply to the setting of batch learning with general function approximation. For example, \citet[Section 5.2]{munos2007performance} points out that AVI enjoys linear-in-horizon error propagation if it \emph{happens to} converge.\footnote{Our paper provides a novel explanation of this result: when FQI (which is a concrete instantiation of the abstract AVI procedure) happens to converge, \citet{chen2019information} shows that its solution coincides with that of \algsq, which we show enjoys linear-in-horizon error propagation whatsoever.} Unfortunately, AVI---and iterative methods in general---has no convergence guarantees (and known to diverge with simple linear classes) unless used with very restricted choices of function approximators~\citep[see e.g., averagers;][]{gordon1995stable}. As another example, linear-in-horizon error can be achieved if one can directly minimize the Bellman error~\citep[e.g.,][]{geist2017bellman}, but computing that requires knowledge of the transition probabilities.  We refer the readers to~\citet{scherrer2014approximate} and the references therein for further results of this kind. 

The only exceptions we are aware of are the non-stationary versions of AVI/API~\citep[e.g.,][]{scherrer2012use}, when the algorithm is allowed to output a periodic non-stationary policies consisting of $\Omega(\nicefrac{1}{(1-\gamma)})$ stationary policies. For a typical value of $\gamma=0.99$ this translates to $100$ policies, and we believe such a complexity is responsible for the clever idea not being picked up in practice despite its appealing theoretical properties. In contrast, we establish linear-in-horizon guarantees for batch algorithms that output simple stationary policies. 

\para{Clean and Tight Concentrability Coefficients} The situation of concentrability coefficients is very similar. The best definition is $\|\wpi{\pi^\star}\|_\infty$, enjoyed by e.g., CPI~\citep{kakade2002approximately} (see also~\citet{agarwal2019optimality}). However, concrete instantiations of these abstract algorithms (in a way that preserve their theoretical properties) typically require on-policy Monte-Carlo roll-outs, which are not available in the batch setting. The same constant has been associated with an abstract Bellman error minimization procedure~\citep{geist2017bellman}, but the algorithm only searches over valid value-functions (instead of arbitrary functions produced by the function approximator). While our definition is worse than theirs by a maximum over policies under consideration, it is still significantly tighter and cleaner than the per-step definitions in most previous analyses of AVI/API~\citep{szepesvari2005finite, munos2007performance, antos2008learning, farahmand2010error}. In fact, we show in Appendix~\ref{sec:per_step_con} that even in a simple uncontrolled setting, our occupancy-based definition can be $\nicefrac{1}{(1-\gamma)}$ multiplicatively tighter than \emph{any} per-step definitions. 

\para{\algsq} The first algorithm we analyze, \algsq, is essentially the analogy of Modified BRM~\citep{antos2008learning} (which approximates $Q^\pi$) in the context of approximating $Q^\star$. To our knowledge, the algorithm is first analyzed by~\citet{chen2019information}, and we improve their loss bound by $\nicefrac{1}{(1-\gamma)}$ (which translates to $\nicefrac{1}{(1-\gamma)^2}$ improvement in sample complexity). It is also worth pointing out that~\citet{dai2018sbeed} has derived a closely related algorithm and demonstrated its empirical effectiveness with deep neural nets. 

\para{\algavg} Our second algorithm, \algavg, is presented and described in such a general form for the first time. That said, the algorithmic idea can be found in several recent works: Just as \algsq is the $Q^\star$-counterpart of Modified BRM, \algavg is the $Q^\star$-counterpart of the MQL algorithm for off-policy evaluation~\citep{uehara2019minimax}. Another closely related work is kernel loss~\citep{feng2019kernel}, which becomes similar to \algavg when the implicit maximization in the RHKS is interpreted as searching over an importance weight class (this connection is pointed out by~\citet{uehara2019minimax}). Finally, the average Bellman error is first used by~\citet{jiang2017contextual} for PAC-exploration with function approximation, and \algavg can be viewed as the batch analogy of their OLIVE algorithm, using importance weights to mimic the data collected by different exploration policies. 


\section{Telescoping Performance Difference}
\label{sec:telepd}
We present the important telescoping lemmas that enable the nice guarantees of the algorithms to be introduced and analyzed later. 
We start with a simple telescoping lemma, which has also been used in recent off-policy evaluation literature~\citep[e.g.,][]{uehara2019minimax}.  
Unless otherwise specified, the full proofs of the results in the main text can be found in Appendix \ref{sec:proofs}.
\begin{lemma}\label{lem:telscp}
For any policy $\pi$ and any $Q\in \RR^{\Scal\times\Acal}$,
\begin{align}
\E_{d_0}[Q(s,\pi)] - J(\pi) = \frac{\E_{d_\pi}\left[Q(s,a) - r - \gamma Q(s',\pi)\right]}{1 - \gamma}.
\end{align}
\end{lemma}
\begin{proof}[Proof Sketch]
$J(\pi) = \frac{\E_{d_\pi}[r]}{1-\gamma}$, so we can remove them from both sides. The remaining terms cancel out by telescoping, which is essentially  the Bellman equation for $d_\pi$ found in the dual linear program of MDPs.
\end{proof}
Using this lemma, we prove the following performance difference bound, which is central to the nice guarantees we are able to prove for \algsq and \algavg. The coarse-grained, $\ell_{\infty}$ version of Theorem~\ref{thm:telepd} for the specific choice of $\pi = \pi^\star$ has been given by~\citet{williams1993tight}, and some of the technical insights can be found in the derivations of~\citet{munos2007performance}. Still, we present the results in a general and agnostic fashion, and their applications to the analyses of   \algsq and \algavg are also novel. 
\begin{theorem}[Telescoping Performance Difference]
\label{thm:telepd}
For any policy $\pi$ and any $Q\in \RR^{\Scal\times\Acal}$,
\begin{align}
J(\pi) - J(\pi_Q) \leq &~ \frac{\Ebb_{d_{\pi}}\left[\Tcal Q - Q\right]}{1 - \gamma} + \frac{\Ebb_{d_{\pi_Q}}\left[Q - \Tcal Q\right]}{1 - \gamma}.
\end{align}
\end{theorem}
\begin{proof}[Proof Sketch]
Note that
$
J(\pi) - J(\pi_Q) 
\le J(\pi) - \E_{s \sim d_0}[Q(s,\pi)] + \E_{s \sim d_0}[Q(s,\pi_Q)] - J(\pi_Q),
$ 
as the sum of the two terms added on the RHS is non-negative due to greediness of $\pi_Q$. Invoking Lemma~\ref{lem:telscp} on $Q$ with $\pi$ and $\pi_Q$, respectively, yields $\E_{d_\pi}[\Tcal^\pi Q - Q]$ and $\E_{d_{\pi_Q}}[Q - \Tcal^{\pi_Q}Q]$ (up to a horizon factor). These policy-specific Bellman errors can be bounded by the optimality error using the greediness of $\pi_Q$. 
\end{proof}
As the result shows, the difference between $J(\pi_Q)$ and that of any $\pi$ is controlled by the \emph{average} Bellman errors $\E_{(\cdot)}[\Tcal Q - Q]$ under the distributions $d_{\pi}$ and $d_{\pi_Q}$, with only \emph{one factor of horizon}. This is in sharp contrast to the typical analyses for AVI sketched in the introduction (Eq.\eqref{eq:singhyee}), and immediately hints at a linear-in-horizon error propagation for algorithms that control (an upper bound) of the average Bellman errors, and we only need to consider $d_\pi$ and $d_{\pi_Q}$ when characterizing distribution shift effects. In Appendix~\ref{sec:iter_lack_berr}, we also illustrate that iterative methods (such as FQI) fail to control the Bellman error---which is in contrary to the popular folklore belief that they do---and explain in part their quadratic dependence on horizon. 

In addition, the average Bellman errors $\E_{d_\pi}[\Tcal Q - Q]$ do \emph{not} have absolute values inside the expectation, and the errors at different $(s,a)$ pairs with opposite signs may cancel with each other. This property is often ignored in previous works, as they add absolute values (and use Jensen's to bound $\ell_1$ with $\ell_2$ norms) anyway when analyzing algorithms that optimize squared-loss, just as we will do to \algsq. However, we emphasize that it is important to state this theorem in such a primitive form for the analysis of \algavg, which directly estimates such average Bellman errors (allowing sign cancellations) using importance weights. Any absolute value relaxations~\citep[e.g.,][]{williams1993tight} will immediately make the result useless for \algavg. 

We conclude this section with some useful corollaries of Theorem~\ref{thm:telepd}, which may also be of independent interest on their own. 
\begin{corollary}[Two-side Performance Difference Bound]
\label{cor:2sidepd}
For any $Q, f \in \RR^{\Scal\times\Acal}$,
\begin{align}
\left|J(\pi_f) - J(\pi_Q)\right| \leq 2\max\left\{ \frac{\Ebb_{d_{\pi_f}}\left[\Tcal Q - Q\right]}{1 - \gamma} + \frac{\Ebb_{d_{\pi_Q}}\left[Q - \Tcal Q\right]}{1 - \gamma}, \frac{\Ebb_{d_{\pi_Q}}\left[\Tcal f - f\right]}{1 - \gamma} + \frac{\Ebb_{d_{\pi_f}}\left[f - \Tcal f\right]}{1 - \gamma}\right\}.
\end{align}
\end{corollary}

\begin{corollary}[Performance Loss w.r.t.~a Class]
\label{cor:pistarpd}
$\forall Q\in\Qcal$,
\begin{align}
\max_{\pi \in \Pi_{\Qcal}} J(\pi) - J(\pi_Q) \leq &~ \frac{2 \max_{\pi \in \Pi_{\Qcal}}\left|\Ebb_{d_\pi} \left[\Tcal Q - Q\right]\right|}{1 - \gamma}.
\end{align}
\end{corollary}


\section{\algsqfull (\algsq)}
\label{sec:fqi}

We present the performance guarantee of the first algorithm, \algsq, which uses another helper class $\Fcal \subset [0, \Vmax]^{\Scal\times\Acal}$ to model $\Tcal Q$ for any $Q\in\Qcal$, seeking to form an (approximately) unbiased estimate of the Bellman error $\|Q-\Tcal Q\|_{2,\mu}^2$:
\begin{align}
\label{eq:minimaxobj}
\Qhat = \argmin_{Q \in \Qcal}\max_{f \in \mathcal F} \left(\ell_{\mathcal D}(Q;Q) - \ell_{\mathcal D}(f;Q)\right),
\end{align}
where $ \ell_{\mathcal D}(\cdot;\cdot)$ is defined in Eq.\eqref{eq:fqiobj}. To give some intuitions, $\ell_{\Dcal}(Q;Q)$ over-estimates $\|Q-\Tcal Q\|_{2,\mu}^2$ (which is why the double sampling trick was invented in the first place~\citep{baird1995residual}), and the amount of over-estimation can be captured by $\min_{f\in\Fcal} \ell_{\Dcal}(f;Q)$ if $\Fcal$ is a rich function class satisfying $\Tcal Q \in \Fcal,\,\forall Q\in\Qcal$; see~\citet{antos2008learning, chen2019information} for further intuitions.  

We now state the guarantee of the algorithm.




\begin{theorem}[Improved error bound of \algsq]
\label{thm:minimax_FQI}
Let $\Qhat$ be the output of \algsq. W.p.~at least $1-\delta$,
\begin{align}
\max_{\pi \in \Pi_{\Qcal}} J(\pi) - J(\pi_{\Qhat}) \leq &~ \frac{2\sqrt{2 \Ceff}}{1 - \gamma} \left( \sqrt{\epQ} + \sqrt{\epQF} \right) \\
&~ + \frac{\sqrt{\Ceff}}{1 - \gamma} \mathcal O \left(\sqrt{\frac{V_{\max}^2 \ln \frac{|\Qcal| |\mathcal F|}{\delta}}{n}} + \sqrt[4]{\frac{V_{\max}^2 \ln \frac{|\Qcal|}{\delta}}{n}\epQ} + \sqrt[4]{\frac{V_{\max}^2 \ln \frac{|\Qcal| |\mathcal F|}{\delta}}{n}\epQF}\right),
\end{align}
where
\begin{align}
\Ceff \coloneqq  &~ \max_{\pi\in\PiQ} \|\wpi{\pi}\|_{2,\mu}^2. \\
\epQ \coloneqq &~ \min_{Q \in \Qcal} \|Q - \Tcal Q\|_{2,\mu}^2. \\
\epQF \coloneqq &~ \max_{Q \in \Qcal}\min_{f \in \mathcal F} \|f - \Tcal Q\|_{2,\mu}^2.
\end{align}
\end{theorem}
This result improves over the bound of~\citet{chen2019information} in several aspects, which we explain below. Furthermore, their bound for \algsq is structurally the same as that for FQI when $\Fcal$ is set as $\Qcal$, and while we are able to improve the bound for \algsq, some of the improvements cannot be enjoyed by FQI (see the argument of~\citet{scherrer2012use}), creating a gap between performance guarantees of the two algorithms.

In the rest of this section, we explain the result and discuss its significance in detail. We also include a high-level sketch of the proof at the end, deferring the full proof to Appendix~\ref{sec:proofs}. 

\subsection{Errors Terms and Optimality} $\epQ$ measures the violation of the realizability assumption $Q^\star\in\Qcal$, and when the assumption holds exactly we have $\epQ=0$ as $\|Q^\star - \Tcal Q^\star\| =0$. Similarly, $\epQF$ measures the violation of the assumption that $\Tcal Q \in \Fcal, \forall Q\in\Qcal$. These definitions are directly taken from~\citet{chen2019information} and consistent with prior literature \citep[e.g.,][]{antos2008learning}. 
The statistical error term within $\Ocal(\cdot)$ is also the same as~\citet{chen2019information}, which consists of a $n^{-\nicefrac{1}{2}}$ fast rate term and two $n^{-\nicefrac{1}{4}}$ terms which vanish as the approximation errors $\epQ$ and $\epQF$ go to $0$. The novelty of the bound is in the multiplicative constants in front of these errors. 

Regarding the optimality guarantee (LHS of the bound), note that we compete with $\max_{\pi \in \PiQ} J(\pi)$ as the optimal value. Slightly modifying the analyses will immediately allow us to compete with any policy $\pi$ even if it is not in $\Pi_Q$ (e.g., $\pi^\star$), as long as we include the policy in the definition of $\Ceff$. 

\subsection{Concentrability Coefficient} \label{sec:con_coeff}
The distribution shift effects are characterized by $\Ceff$ in our bound. Not only this definition is much simpler, it is also tighter than previous definitions in two ways, and we start with the minor one: we use a weighted square of $\wpi{\pi}$ rather than its $\ell_\infty$ norm, the latter of which is more common in literature~\citep{munos2007performance, munos2008finite, antos2008learning, chen2019information}. It is easy to show that the squared version is tighter~\citep{farahmand2010error}: for example, consider the $\ell_\infty$ version of our $\Ceff$, which should be defined as
\begin{align}
\Cinf \coloneqq \max_{\pi \in \PiQ} \|\wpi{\pi}\|_\infty.
\end{align}
One can easily show that $\Ceff$ is tighter: for any $\pi \in \PiQ$,
\begin{align}
\|\wpi{\pi}\|_{2,\mu}^2 = \E_{\mu}[\wpi{\pi}^2] 
\le \E_{\mu}[\Cinf \wpi{\pi}] = \Cinf.
\end{align}
The second improvement, which is much more significant, is the departure from ``per-step'' definitions. In all analyses of AVI/API, the concentrability coefficient takes the form of
\begin{align}\label{eq:Cps}
\Cps \coloneqq \sum_{t=0}^\infty \beta(t) C_t,~~ C_t \coloneqq \max_{\pi} \|\wpi{\pi,t}\|_\infty,
\end{align}
where $d_{\pi,t}$ is the marginal distribution
of $(s_t,a_t)$. $\beta(t)$ is a series of non-negative coefficients that sum up to $1$. Different versions of $\Cps$ differ in $\beta(t)$, the policy space considered in $\max_\pi$ (typically non-stationary policies concatenated using policies from $\PiQ \bigcup \{\pi^\star\}$), and sometimes replacing $\|\cdot\|_\infty$ with $\|\cdot\|_2^2$; see~\citet{farahmand2010error} for a detailed discussion. 
While it is difficult to directly compare this quantity to ours due to its complication, we show that in a simplest uncontrolled scenario where there is no distribution shift at all, \emph{any} per-step definition will be at least $1/(1-\gamma)$ looser than ours. We include an intuitive but informal statement below, and defer the detailed discussions to Appendix~\ref{sec:per_step_con}. 
\begin{proposition}[Informal] \label{prop:per_step_con}
Consider an uncontrolled deterministic problem (there is only 1 action) formed by a long chain of states. Let $\mu = d_\pi$ where $\pi$ is the only policy. $\Cinf = \Ceff = 1$, and any definition of $\Cps \ge 1/(1-\gamma)$.
\end{proposition}

\subsection{Horizon Dependence} \label{sec:msbo_linear}
We now verify that the bound has linear dependence on horizon. Doing so can be tricky given the complicated expression, and we provide 3 verification methods following the conventions in the literature~\citep{scherrer2014approximate}: The first one is to observe that FQI has quadratic dependence on horizon and our bound for \algsq has a $\nicefrac{1}{(1-\gamma)}$ net improvement over FQI~\citep{chen2019information}. The second one is to read the expression, and count the explicit dependence; 
while the statistical error depends on $\Vmax = R_{\max}/(1-\gamma)$, such a dependence is superficial and not produced by error accumulation over multi-stage decision-making, and is never counted in the literature.\footnote{See~\citet{jiang2018open} for a deeper discussion.} 
The third method is to consider the fully realizable case ($\epQ = \epQF =0$) and calculate the sample complexity. Since the statistical rate is $1/\sqrt{n}$, an algorithm with linear-in-horizon error propagation should have $\Ocal(1/(1-\gamma)^2)$ sample complexity, which we show below. This contrasts the $\Ocal(1/(1-\gamma)^4)$ sample complexity of FQI~\citep{chen2019information}.
\begin{corollary}[Improved sample complexity of \algsq]
Let $\epQ = \epQF=0$. For any $\epsilon,\delta >0$, Eq.\eqref{eq:minimaxobj} satisfies $\max_{\pi \in \Pi_{\Qcal}} J(\pi) - J(\pi_{\Qhat}) \leq \varepsilon \cdot V_{\max}$ w.p.~$\ge 1 - \delta$, if
\begin{align}
n = \mathcal O \left( \frac{\Ceff \ln \frac{|\Qcal||\mathcal F|}{\delta}}{\varepsilon^2 (1 - \gamma)^2} \right).
\end{align}
\end{corollary}

\subsection{Proof Sketch} \label{sec:msbo_sketch}
We sketch the high-level proof here, deferring the details to Appendix~\ref{sec:proofs}; this analysis is relatively straightforward due to existing work (compared to \algavg, which is novel). To bound $J(\pi) - J(\pi_{\Qhat})$ for any $\pi \in \PiQ$, we invoke Theorem~\ref{thm:telepd}, which produces two average Bellman error terms of form $|\E_{d_\pi}[\Tcal \Qhat - \Qhat]|$. Then
\begin{align}
|\E_{d_\pi}[\Tcal \Qhat - \Qhat]| = |\E_{\mu}[\wpi{\pi}\cdot(\Tcal \Qhat - \Qhat)]| \le \sqrt{\E_{\mu}[\wpi{\pi}^2] \E_{\mu}[(\Tcal \Qhat- \Qhat)^2]} 
\le \sqrt{\Ceff} \|\Tcal \Qhat - \Qhat \|_{2, \mu}.
\end{align}
The last step follows from Cauchy-Schwarz for random variables, and the term $\|\Tcal \Qhat - \Qhat \|_{2, \mu}$ is well-studied by~\citet{chen2019information} and we directly use their result.


\section{\algavgfull (\algavg)}
\label{sec:mabo}
We introduce and analyze our second (and novel) algorithm, \algavg, which directly estimates the average Bellman errors (allowing sign cancellations) that show up in the telescoping results from Section~\ref{sec:telepd} by explicit importance-weighting correction. Doing so requires an additional function approximator $\Wcal$ to model the marginalized importance weights (see Section~\ref{sec:vfa}), $\Wcal \subset \RR^{\Scal\times\Acal}$, in addition to the $\Qcal$ class that models $Q^\star$. Given $\Qcal$ and $\Wcal$, the algorithm is
\begin{align}
\label{eq:MABO}
\Qhat = \argmin_{Q  \in \Qcal} \max_{w \in \Wcal} \quad \left|\mathcal L_{\mathcal D} (Q,w)\right|,
\end{align}
where
\begin{align}
\Lcal_{\Dcal} (Q,w) \coloneqq \E_{\Dcal} \left[w(s,a) \left(r + \gamma \max_{a'}Q(s',a') - Q(s,a) \right)\right].
\end{align}
It is important to point out that we only use the single sample estimate of Bellman error (i.e., no double sampling), but we obtain an unbiased estimate of average Bellman error thanks to not using the squared loss (unlike $\ell_{\Dcal}(Q;Q)$ in \algsq, which is an over-estimation). To see how $\Lcal_{\Dcal}(Q,w)$ is related to the average Bellman errors, simply consider its population version: 
\begin{align} \label{eq:mabo_popu}
\mathcal L_{\mu} (Q,w) \coloneqq \E_{\Dcal} [\Lcal_{\Dcal}(Q,w)] \nonumber =  \E_{(s,a) \sim \mu} \left[w(s,a) \left((\Tcal Q)(s,a) - Q(s,a) \right) \right],
\end{align}
thus $\Lcal_\mu(Q, \wpi{\pi}) = \E_{d_\pi}[\Tcal Q - Q]$. Therefore, as long as $\Wcal$ realizes $\wpi{\pi}$ for all $\pi\in\PiQ$ (this assumption will be relaxed), $\max_{w\in\Wcal} |\Lcal_\mu(Q,w)|$ will control the suboptimality gap of $\pi_Q$, which is the intuition for the algorithm. 

We now state the guarantee of this algorithm. For convenience, we will use $\E_{\mu}[w\cdot(\Tcal Q - Q)]$ as a shorthand for Eq.\eqref{eq:mabo_popu} in the rest of this paper. 
\begin{theorem}[Error bound of MABO]
\label{thm:batch_olive_bound}
Let $\Qhat$ be the output of \algavg. W.p.~$1 - \delta$, 
\begin{align}
\max_{\pi \in \Pi_{\Qcal}} J(\pi) - J(\pi_{\Qhat}) \leq \frac{2}{1 - \gamma}\left(\epsQ + \epsW + \epsstat \right).
\end{align}
where
\begin{align}
\epsQ \coloneqq &~ \min_{Q \in \Qcal} \max_{w \in \Wcal}\left|\E_{\mu}[w\cdot(\Tcal Q - Q)]\right|, \\
\epsW \coloneqq &~ \max_{\pi \in \Pi_{\Qcal}} \inf_{w \in \lspan(\Wcal)} \max_{Q \in \Qcal} \bigg|\E_{\mu} \big[(\wpi{\pi} - w) \cdot(\Tcal Q - Q)\big]\bigg|, \\ 
\epsstat \coloneqq &~ 2 V_{\max} \sqrt{\frac{2 \CeffW \ln\frac{2|\Qcal||\Wcal|}{\delta}}{n}} + \frac{4 \CinfW V_{\max} \ln\frac{2|\Qcal||\Wcal|}{\delta}}{3n},\\
\CeffW \coloneqq &~ \max_{w\in\Wcal} \|w\|_{2, \mu}^2,  \quad
\CinfW \coloneqq \max_{w\in\Wcal} \|w\|_\infty,
\end{align}
and 
$\lspan(\Wcal)$ is the linear span of $\Wcal$ using coefficients with (at most) unit $\ell_1$ norm, i.e., 
$$ \textstyle 
\lspan(\Wcal) \coloneqq \left\{\sum_{w \in \Wcal} \alpha(w) w : \sum_{w\in\Wcal}|\alpha(w)| \le 1\right\}.$$ 
\end{theorem}
%
%
%
%

In the rest of this section, we explain the bound and discuss its significance.

\subsection{Error Terms and Augmented Expressivity} \label{sec:improper}
Similar to $\epQ$ for \algsq, $\epsQ$ also measures the violation of $Q^\star \in \Qcal$, though in a different manner: we measure $Q$'s worst-case average Bellman error on any $w \in \Wcal$.

The situation of $\epsW$ is a little more special. Despite that we provide intuition for \algavg by requiring that $\wpi{\pi_Q}\in \Wcal, \forall Q\in\Qcal$, it turns out we only need a much more relaxed version of this assumption (and can measure violation against the relaxed version): thanks to the linearity of $\Lcal_{\Dcal}(Q,\cdot)$, we are automatically approximating $\wpi{\pi_Q}$ from an augmented class $\lspan(\Wcal)$.\footnote{Similar properties have been recognized regarding the policy evaluation counterpart of \algavg~\citep{uehara2019minimax}.} Moreover, the loss $\Lcal_{\Dcal}(Q,w)$ is ``scale-free'' w.r.t.~$w$, i.e., it is completely equivalent to replace $\Wcal$ with any $c \Wcal \coloneqq \{cw : w\in\Wcal\}$, for any $c\ne 0$. Therefore, we may rescale $\Wcal$ arbitrarily in the theorem to obtain the sharpest bound. 

To help develop further intuition, we illustrate the idea using a familiar tabular example: Consider the case where $|\Scal|$ and $|\Acal|$ are manageable and we use a tabular function class $\Qcal \coloneqq [0, \Vmax]^{\Scal\times\Acal}$. It is easy to see that we can recover the standard tabular model-based algorithm (a.k.a.~\emph{certainty-equivalence}, or C-E) by using $\Wcal = \{(s,a) \mapsto \mathds{1}(s=s^*, a=a^*): s^*\in \Scal, a^* \in \Acal\}$, i.e., a set of $|\Scal\times\Acal|$ indicator functions. This is because the lowest possible value for the objective is $0$, achieving which requires that $|\Lcal_\Dcal(Q,w)|=0, \forall w \in \Wcal$. This set of $|\Wcal|=|\Scal\times\Acal|$ equations is essentially the Bellman equation for each state-action pair in the empirical MDP, which can and can only be satisfied by the C-E solution. While the C-E solution incurs no approximation error, $\Wcal$ clearly fails to realize $\wpi{\pi_Q}$ for all $Q\in\Qcal$. The reason, as we have already explained earlier, is because $\lspan(\Wcal)$---which now becomes the tabular function space---can model any importance weights with proper scaling.

As a final remark, given any $w\in\lspan(\Wcal)$ and the target importance weight $\wpi{\pi_Q}$, we measure their distance by projecting their difference using $\Tcal Q - Q$ for the worst-case $Q \in \Qcal$. If we treat it as approximating distribution $d_\pi$ with $(\mu \,\cdot\, w)(s,a):= \mu(s,a) w(s,a)$, then this measure is essentially the Integral Probability Metric~\citep{muller1997integral} between $d_\pi$ and $\mu\,\cdot\, w$ using a discriminator class induced by $\Qcal$.

\subsection{Concentrability Coefficients} \label{sec:CW}
Our $\CeffW$ and $\CinfW$ are defined in a way similar to $\Ceff$ and $\Cinf$ in Section~\ref{sec:fqi}, except that we consider $w\in \Wcal$, i.e., the functions provided by the function approximator $\Wcal$ instead of the true importance weights $\wpi{\pi_Q}$ themselves. While these two sets of coefficients are not directly comparable, we provide some insights about their relationship. 

On one hand, if we choose $\Wcal = \{\wpi{\pi_Q}: Q\in\PiQ\}$, which precisely satisfies the expressivity assumption, then $\CeffW = \Ceff$ and $\CinfW = \Cinf$. Given that $\Wcal$ is likely to include other functions as well, we might conclude that $\CeffW$ and $\CinfW$ are in general greater. On the other hand, to satisfy $\epsW=0$ we only need $\lspan(\Wcal)$ to be the above-mentioned class, and the actual $\Wcal$ could be smaller and simpler. Also, since $\CeffW$ and $\CinfW$ only occur in the statistical error term in Theorem~\ref{thm:batch_olive_bound} (which is in sharp contrast to Theorem~\ref{thm:minimax_FQI}, where $\Ceff$ also amplifies approximation errors), the damage caused by $w \in \Wcal$ with unnecessarily large magnitude can be mitigated by proper regularization (see e.g., \citet{kallus2016generalized, hirshberg2017augmented, su2019doubly} for how importance weights can be regularized in contextual bandits). Given these competing considerations, we suggest that it is reasonable to treat $\CeffW \approx \Ceff$, $\CinfW \approx \Cinf$. 

\subsection{Horizon Dependence}
The linear dependence on horizon of Theorem~\ref{thm:batch_olive_bound} can be verified in a way similar to Section~\ref{sec:msbo_linear}, and we only include the sample complexity of \algavg when all the expressivity assumptions are met exactly. The sample complexity contains two terms corresponding to the slow rate ($n^{-\nicefrac{1}{2}}$) and the fast rate ($n^{-1}$) terms in $\epsstat$, and when $\CinfW$ is not too much larger than $\CeffW$,\footnote{E.g., $\CeffW = \CinfW$ when $\Wcal$ only contains indicator functions (e.g., in the tabular scenario in Section~\ref{sec:CW}).} the fast rate term is dominated and the sample complexity is very similar to that of \algsq.
\begin{corollary}[Sample complexity of \algavg]
\label{cor:mabo_spcplx}
Suppose $\epsQ = \epsW=0$. The output of \algavg Eq.\eqref{eq:MABO}, satisfies $\max_{\pi \in \Pi_{\Qcal}} J(\pi) - J(\pi_{\Qhat}) \leq \varepsilon \cdot V_{\max}$ ~w.p.~$1 - \delta$, if
\begin{align}
n = \mathcal O \left(\left(\frac{\CeffW}{\varepsilon^2 (1 - \gamma)^2} + \frac{\CinfW}{\varepsilon (1-\gamma)} \right) \ln \frac{|\Qcal||\Wcal|}{\delta}\right).
\end{align}
\end{corollary}

\subsection{Proof Sketch of Theorem \ref{thm:batch_olive_bound}}
We conclude the section by a high-level proof sketch. With Theorem~\ref{thm:telepd}, it suffices to control 
$|\E_{d_\pi}[\Tcal \Qhat - \Qhat]|$ $= |\E_{\mu}[ \wpi{\pi} \cdot(\Tcal \Qhat - \Qhat)]|$ for the worst-case $\pi\in\PiQ$. Fixing any $\pi$, the first step is to peel off the approximation error of $\Wcal$:  for any $w\in\lspan(\Wcal)$, we have
\begin{align}
&~ |\E_{\mu}[\wpi{\pi} \cdot(\Tcal \Qhat - \Qhat)]| \\
\le &~ |\E_{\mu}[(\wpi{\pi} - w)  (\Tcal \Qhat - \Qhat)]| + |\E_{\mu}[w \cdot(\Tcal \Qhat - \Qhat)]|  \\
\le &~ \max_{Q \in \Qcal}|\E_{\mu}[(\wpi{\pi} - w)(\Tcal Q - Q)]| + |\E_{\mu}[w\cdot(\Tcal \Qhat - \Qhat)]|.
\end{align}
So if we choose $w$ as the one that achieves the infimum in the definition of $\epsW$, denoted as $\widehat w$, then the first term is bounded by $\epsW$. The second term is much closer to the loss function of \algavg, and can be handled as
\begin{align*}
|\E_{\mu}[ \widehat w \cdot (\Tcal \Qhat - \Qhat)]| \le &~ \sup_{w \in \lspan(\Wcal)}|\E_{\mu}[ w\cdot (\Tcal \Qhat - \Qhat)]| \\
=  &~ \max_{w \in \Wcal}|\E_{\mu}[w \cdot (\Tcal \Qhat - \Qhat)]|.
\end{align*}
Crucially, using the linearity of $\E_{\mu}[w\cdot(\cdot)]$ in $w$ and the norm constraints of $\lspan(\cdot)$, we are able to replace $\sup_{w\in\lspan(\Wcal)}$ with $\max_{w\in\Wcal}$, leading to the augmented expressivity discussed in Section~\ref{sec:improper}; see Eq.\eqref{eq:mabo_span_bound} in Appendix~\ref{sec:proofs} for a detailed argument. Then with similar strategies, we peel off the approximation error of $\Qcal$ from $|\E_{\mu}[\widehat w \cdot (\Tcal \Qhat - \Qhat)]|$. The rest of the analysis handles statistical errors using generalization error bounds.

\section{Further Comparisons and Discussions}
\label{sec:discuss}
In the previous sections we have analyzed \algsq and \algavg, showing that they enjoy linear-in-horizon error propagation and cleanly and tightly defined concentrability coefficients, which answers \qhor and \qcon in the introduction. Still, \algsq bears significant similarities to classical AVI/API algorithms\footnote{Recall that FQI coincides with \algsq using $\Fcal = \Qcal$ when FQI converges~\citep{chen2019information}, and in this sense \algsq can be viewed as a best-case scenario for FQI.} in the use of squared loss and the expressivity requirement on function approximation (\qfun and \qsq). In this section we compare its guarantee (Theorem~\ref{thm:minimax_FQI}) to that of \algavg (Theorem~\ref{thm:batch_olive_bound}), and discuss the potential advantages of \algavg (which is novel and understudied), as well as its limitations, compared to currently popular algorithms. The recurring theme of the comparisons---as we will see below---is the pros and cons of implicit (e.g., FQI and \algsq) and explicit (\algavg) distribution corrections. 

\subsection[Robustness Against Misspecified Q]{Robustness Against Misspecified $\Qcal$}
We compare the robustness of the two algorithms against misspecified $\Qcal$, that is, how much we pay when $Q^\star \notin \Qcal$. Omitting the common horizon factor, \algsq pays $\Ocal(\sqrt{\Ceff \cdot \epQ})$ and \algavg pays $\Ocal(\epsQ)$. Again, they are not directly comparable, but we can still offer some useful insights. Imagine the scenario of $\Wcal = \{\wpi{\pi_Q} : Q\in\Qcal\}$ (as we did in Section~\ref{sec:CW}), then
\begin{align} 
\label{eq:epQ_epsQ}
&~\epsQ = \min_{Q \in \Qcal} \max_{\pi \in \PiQ}\left|\E_{\mu}[\wpi{\pi} \cdot(\Tcal Q - Q)]\right|  \\
\le &~ \min_{Q \in \Qcal} \max_{\pi \in \PiQ}\sqrt{\E_{\mu}[\wpi{\pi}^2] \cdot \E_{\mu}[(\Tcal Q - Q)^2]} \nonumber 
=  \sqrt{\Ceff \cdot \epQ}.   
\end{align}
Here the second step follows from Cauchy-Schwarz, which we also used in Section~\ref{sec:msbo_sketch}. As we can see, if $\Wcal$ is specified ``just right'', \algavg's guarantee never suffers more than that of \algsq on misspecified $\Qcal$, and any looseness from Cauchy-Schwarz\footnote{See \qsq in the introduction.} enters the gap. On the other hand, such an advantage of \algavg may be weakened if $\Wcal$ includes additional functions that do not correspond to real importance weights.  

Another difference between \algsq and \algavg is that \algsq pays $\sqrt{\Ceff}$ in front of $\sqrt{\epQ}$, whereas \algavg does not pay any concentrability coefficients in its approximation error terms, thanks to explicit distribution correction. While Eq.\eqref{eq:epQ_epsQ} might leave the impression that the difference is superficial,  the inequality only relaxes $\epsQ$ (apart from the nice choice of $\Wcal$) hence unfairly favors \algsq, and there are scenarios where the $\sqrt{\Ceff}$ difference is real: for example, consider the scenario where $Q$ has uniformly low error across all distributions, and $Q'$ has small Bellman error on $\mu$ but (up to $\sqrt{\Ceff}$ times) higher errors on e.g., $d_{\pi_Q'}$. In this case, \algavg clearly prefers $Q$ over $Q'$ due to explicit distribution correction, whereas \algsq is indifferent between them and can suffer the poor performance of $Q'$.

\subsection{Statistical Rates}
The $n^{-\nicefrac{1}{2}}$ terms in Theorems~\ref{thm:minimax_FQI} and \ref{thm:batch_olive_bound} match each other if we treat $\Ceff \approx \CeffW$ (see Section~\ref{sec:CW}). \algavg suffers another $\CinfW/n$ term, whereas $\Cinf$ does not enter the guarantee of \algsq; this is an (unfortunately) inevitable consequence of explicit importance weighting and concentration inequalities. On the other hand, the term fades away quickly with $n$ and will be of minor issue with sufficient data. Finally, \algsq suffers two $n^{-\nicefrac{1}{4}}$ terms, and although they can be absorbed by the worse between the fast rate and the approximation error terms in Big-Oh notations~\citep[Appendix C]{chen2019information}, doing so worsens the constant. 

\subsection{Assumptions on the Helper Classes}
A characteristic shared by \algsq and \algavg is the use of a helper class ($\Fcal$ for \algsq and $\Wcal$ for \algavg) to assist the estimation of the Bellman error. These helper classes also take the heaviest expressivity burdens in their corresponding algorithms: while $\Qcal$ is only required to capture $Q^\star$, $\Fcal$ and $\Wcal$ are required to capture $\Tcal Q$ and $\wpi{\pi_Q}$, respectively, \emph{for all} $Q\in \Qcal$. 

While $\Fcal$ and $\Wcal$ model completely different objects, we note that $\Wcal$ enjoys a superior property that $\Fcal$ does not have, that is we essentially approximate the importance weights from $\lspan(\Wcal)$, allowing simple $\Wcal$ to have high expressivity. This property crucially comes from the linearity of the average Bellman error loss, which is another advantage of the average loss over the squared loss. 

To further illustrate the representation power of $\lspan(\Wcal)$, we provide the following result, showing that in MDPs with low-rank dynamics (which are often sufficient conditions that allow an exploratory\footnote{Technically, a small $\Ceff$ or $\Cinf$.} $\mu$ to exist in the first place~\citep{chen2019information}), there exists very simple (in the sense of low statistical complexity) $\Wcal$ that satisfies $\epsW=0$. 

\begin{proposition}\label{prop:lowrk}
Suppose the rank of the MDP's transition matrix is $k$. Then, 
\begin{enumerate}[leftmargin=*]
\item For any choice of $\Qcal$, there exists $\Wcal$ with cardinality $|\Wcal| \le  (k+1)|\Qcal|$, such that $\epsW = 0$.
\item Let the transition matrix $P = \Phi P'$, where $\Phi \in \RR^{|\Scal\times\Acal|\times k}$ and let $\phi(s,a)^\top$ denote its $(s,a)$-th row. For the choice of $\Qcal = \{(s,a) \mapsto R(s,a) + \gamma \phi(s,a)^\top \theta: \theta \in \RR^k\}$, there exists $\Wcal$ with cardinality  $|\Wcal|\le k+1$ such that $\epsW=0$. 
\end{enumerate} 
\end{proposition}
The formal definitions and proofs are deferred to Appendix~\ref{sec:lowrk}. In the first claim (general case), $\Wcal$ has low statistical capacity despite scaling with $|\Qcal|$, 
as we need to pay $\ln|\Qcal|$ anyway by using the $\Qcal$ class, and the dependence of $|\Wcal|$ on $|\Qcal|$ is not a significant burden. 
In the second claim, which is the more restricted ``linear MDP'' setting recently studied by e.g., \citet{yang2019sample}, we are able to bring $|\Wcal|$ down to as low as $k+1$; it is also interesting to point out that we cannot guarantee $\wpi{\pi_Q} \in \lspan(\Wcal)$, but using the linear structure of $\Qcal$ we can still prove that $\epsW = 0$. 
Finally, we emphasize that the existence of such a simple $\Wcal$ does not imply that we are guaranteed to find it for every problem, as the design of function approximation always requires appropriate prior knowledge and inductive biases.


\section{Conclusions}
\label{sec:con}
We analyze two algorithms, \algsq and \algavg, which enjoy linear-in-horizon error propagation, a property established for the first time for batch algorithms outputting stationary policies. \algavg uses a novel importance-weight correction to handle the difficulty of Bellman error estimation, and our analyses reveal its distinct properties and potential advantages compared to classical squared-loss-based algorithms.

\section*{Acknowledgement}
The authors thank Aditya Modi for providing the references to some important related works. 

\newpage


%% file: sections/table1.tex
\begin{table*}[t]
	\centering
	\caption{Algorithms considered in this paper, all of which require $Q^\star \in \Qcal$ (definitions of approximation error differ). $\Qcal$, $\Fcal$ and $\Wcal$ are defined in Section \ref{sec:vfa}, Section \ref{sec:fqi} and Section \ref{sec:mabo}. 
	}
	\resizebox{\textwidth}{!}{
	\begin{tabular}{|c|c|c|c|c|c|}\hline
		\textbf{Algorithm} & \textbf{Style}
		& \tabincell{c}{\textbf{Requirement on}\\ \textbf{helper class}} & \tabincell{c}{\textbf{Horizon} \\ \textbf{dependence}}
		& \tabincell{c}{\textbf{Concentrability} \\ \textbf{coefficient}}
		& \tabincell{c}{\textbf{Related} \\ \textbf{practical algorithm}}
		\\\hline
		FQI & Iterative + Sq-loss 
		& $\forall Q\in \Qcal, \Tcal Q \in \Qcal$  & $\nicefrac{1}{(1-\gamma)^2}$ 
		& \tabincell{c}{Per-step-based \\ (Eq.\eqref{eq:Cps} and App.\ref{sec:per_step_con})}
		& \tabincell{c}{DQN \\ \citep{mnih2015human}}
		\\ \hline
		\algsq & Minimax + Sq-loss 
		& $\forall Q\in \Qcal,~ \Tcal Q \in \Fcal$ & $\nicefrac{1}{(1-\gamma)}$ 
		& \tabincell{c}{Occupancy-based\\ (see Thm.\ref{thm:minimax_FQI})}
		& \tabincell{c}{SBEED \\ \citep{dai2018sbeed}}
		\\ \hline
		\algavg & Minimax + Avg-loss 
		& \tabincell{c}{$\forall Q\in\Qcal$, \\$\wpi{\pi_Q} \in \lspan(\Wcal)$} &$\nicefrac{1}{(1-\gamma)}$ 
		& $\Wcal$-based (see Thm.\ref{thm:batch_olive_bound})
		& \tabincell{c}{Kernel-loss \\ \citep{feng2019kernel}} \\ \hline
	\end{tabular}}
\end{table*}

%% file: sections/proof_arxiv.tex
\section{Detailed Proofs}
\label{sec:proofs}

\begin{replemma}{lem:telscp}[Evaluation error lemma, restated]
For any policy $\pi$ and any $Q\in \RR^{\Scal\times\Acal}$,
\begin{align}
\E_{d_0}[Q(s,\pi)] - J(\pi) = \frac{\E_{d_\pi}\left[Q(s,a) - r - \gamma Q(s',\pi)\right]}{1 - \gamma}.
\end{align}
\end{replemma}

\begin{proof}[Proof of Lemma \ref{lem:telscp}]
Since $J(\pi) = \frac{\E_{d_\pi}[r]}{1-\gamma}$, we remove these terms from both sides, and prove the rest of the identity. 
\begin{align}
&~ \frac{\E_{(s,a,r,s') \sim d_{\pi}}\left[Q(s,a) - \gamma Q(s',\pi(s'))\right]}{1 - \gamma} \\
= &~ \sum_{s,a} \sum_{t = 0}^{\infty} \gamma^t \Pr(s_t = s, a_t = a|s_0 \sim d_0, \pi) Q(s,a) - \sum_{s,a} \sum_{t = 1}^{\infty} \gamma^t \Pr(s_t = s|s_0 \sim d_0, \pi) Q(s,\pi(s)) \\
= &~ \sum_{s,a} \sum_{t = 0}^{\infty} \gamma^t \Pr(s_t = s, a_t = a|s_0 \sim d_0, \pi) Q(s,a) - \sum_{s,a} \sum_{t = 1}^{\infty} \gamma^t \Pr(s_t = s, a_t = a|s_0 \sim d_0, \pi) Q(s,a) \\
= &~ \sum_{s,a} \Pr(s_0 = s, a_0 = a|s_0 \sim d_0, \pi) Q(s,a) = \E_{s \sim d_0}[Q(s,\pi(s))],
\end{align}
where the first equation follows from the definition of $d_{\pi}$, the second equation follows from the definition of $Q(s,\pi(s))$.
\end{proof}

\begin{reptheorem}{thm:telepd}[Telescoping Performance Difference, restated]
For any policy $\pi$ and any $Q\in \RR^{\Scal\times\Acal}$,
\begin{align}
J(\pi) - J(\pi_Q) \leq &~ \frac{\Ebb_{d_{\pi}}\left[\Tcal Q - Q\right]}{1 - \gamma} + \frac{\Ebb_{d_{\pi_Q}}\left[Q - \Tcal Q\right]}{1 - \gamma}.
\end{align}
\end{reptheorem}
\begin{proof}[Proof of Theorem \ref{thm:telepd}]

\begin{align}
J(\pi) - J(\pi_Q) 
= &~ \underbrace{J(\pi) - \E_{s \sim d_0}[Q(s,\pi(s))]}_{\text{(I)}} + \underbrace{\E_{s \sim d_0}[Q(s,\pi(s))] - \E_{s \sim d_0}[Q(s,\pi_Q(s))]}_{\text{(II)}}\\
&~ + \underbrace{\E_{s \sim d_0}[Q(s,\pi_Q(s))] - J(\pi_Q)}_{\text{(III)}}.
\end{align}
These three terms can be bound separately as follows.
\begin{align}
\text{(I)} = &~ J(\pi) - \E_{s \sim d_0}[Q(s,\pi)]\\
= &~ \frac{1}{1 - \gamma} \E_{d_{\pi}}\left[r + \gamma Q(s',\pi) - Q(s,a)\right] \\
\leq &~ \frac{1}{1 - \gamma} \E_{d_{\pi}}\left[r + \gamma \max_{a'\in\Acal} Q(s',a') - Q(s,a)\right] \\
= &~ \frac{1}{1 - \gamma} \E_{d_{\pi}}\left[\Tcal Q - Q\right].
\end{align}
The second equation follows from Lemma \ref{lem:telscp}, and the last step follows from marginalizing out $r$ and $s'$ by conditioning on $(s,a)$ using law of total expectation. 

For (II),
\begin{align}
\text{(II)} = &~ \E_{s \sim d_0}\left[Q(s,\pi(s))\right] - \E_{s \sim d_0}\left[Q(s,\pi_Q(s))\right] = \E_{s \sim d_0}\left[Q(s,\pi(s)) - \max_a Q(s,a) \right] \leq 0.
\end{align}

Finally, (III), which is handled similarly to (I).
\begin{align}
\text{(III)} = &~ \E_{s \sim d_0}[Q(s,\pi_Q)] - J(\pi_Q)\\
= &~ \frac{1}{1 - \gamma}\E_{d_{\pi_Q}}\left[Q(s,a) - r - \gamma Q(s',\pi_Q)\right] \\
= &~ \frac{1}{1 - \gamma}\E_{d_{\pi_Q}}\left[Q(s,a) - r - \gamma \max_{a'}Q(s',a')\right] \\
= &~ \frac{1}{1 - \gamma}\E_{(s,a,s') \sim d_{\pi_Q}}\left[Q - \Tcal Q\right],
\end{align}
where the third equation follows from the definition of $\pi_Q$ being greedy w.r.t.~$Q$. The result follows by putting all three parts together.
\end{proof}

\begin{reptheorem}{thm:minimax_FQI}[Improved error bound of \algsq, restated]
Let $\Qhat$ be the output of \algsq. W.p.~at least $1-\delta$,
\begin{align}
\max_{\pi \in \Pi_{\Qcal}} J(\pi) - J(\pi_{\Qhat}) \leq &~ \frac{2\sqrt{2 \Ceff}}{1 - \gamma} \left( \sqrt{\epQ} + \sqrt{\epQF} \right) \\
&~ + \frac{\sqrt{\Ceff}}{1 - \gamma} \mathcal O \left(\sqrt{\frac{V_{\max}^2 \ln \frac{|\Qcal| |\mathcal F|}{\delta}}{n}} + \sqrt[4]{\frac{V_{\max}^2 \ln \frac{|\Qcal|}{\delta}}{n}\epQ} + \sqrt[4]{\frac{V_{\max}^2 \ln \frac{|\Qcal| |\mathcal F|}{\delta}}{n}\epQF}\right),
\end{align}
where
\begin{align}
\Ceff \coloneqq  &~ \max_{\pi\in\PiQ} \|\wpi{\pi}\|_{2,\mu}^2. \\
\epQ \coloneqq &~ \inf_{Q \in \Qcal} \|Q - \Tcal Q\|_{2,\mu}^2. \\
\epQF \coloneqq &~ \sup_{Q \in \Qcal}\inf_{f \in \mathcal F} \|f - \Tcal Q\|_{2,\mu}^2.
\end{align}
\end{reptheorem}

\begin{proof}[Proof of Theorem \ref{thm:minimax_FQI}]
We use $\pi^\star$ to denote $\argmax_{\pi \in \Pi_{\Qcal}} J(\pi)$. By applying Theorem \ref{thm:telepd}, we can obtain
\begin{align}
\max_{\pi \in \Pi_{\Qcal}} J(\pi) - J(\pi_{\Qhat}) \leq &~~ \frac{\Ebb_{d_{\pi^\star}}\left[\Tcal \Qhat - \Qhat\right]}{1 - \gamma} + \frac{\Ebb_{d_{\pi_{\Qhat}}}\left[\Qhat - \Tcal \Qhat \right]}{1 - \gamma} \\
= &~ \frac{\E_{\mu}\left[w_{\nicefrac{d_{\pi^\star}}{\mu}} \cdot \left(\Tcal \Qhat - \Qhat\right)\right]}{1 - \gamma} + \frac{\E_{\mu}\left[w_{\nicefrac{d_{\pi_{\Qhat}}}{\mu}} \cdot \left(\Qhat - \Tcal \Qhat\right)\right]}{1 - \gamma} \\
\overset{\text{(a)}}{\leq} &~ \frac{\sqrt{\E_{(s,a) \sim \mu}\left[\left(w_{\nicefrac{d_{\pi^\star}}{\mu}}(s,a)\right)^2\right]\E_{(s,a) \sim \mu}\left[\left((\Tcal \Qhat) (s,a) - \Qhat(s,a)\right)^2\right]}}{1 - \gamma} \\
&~ + \frac{\sqrt{\E_{(s,a) \sim \mu}\left[\left(w_{\nicefrac{d_{\pi_{\Qhat}}}{\mu}}(s,a)\right)^2\right]\E_{(s,a) \sim \mu}\left[\left((\Tcal \Qhat) (s,a) - \Qhat(s,a)\right)^2\right]}}{1 - \gamma}\\
\label{eq:bound_J_blm}
\overset{\text{(b)}}{\leq} &~ \frac{2 \sqrt{C_{\text{eff}}}}{1 - \gamma} \left\|Q - \mathcal T Q\right\|_{2,\mu}.
\end{align}
where (a) follows from the Cauchy-Schwarz inequality for random variables ($|\mathbb E X Y| \leq \sqrt{\mathbb E[X^2] \mathbb E [Y^2]}$) and (b) follows from the definition of $\Ceff$.

We then directly adopt the upper bound on $\left\|\Qhat - \mathcal T \Qhat\right\|_{2,\mu}$ from \citet{chen2019information}:
\begin{align}
&~ \left\|\Qhat - \mathcal T \Qhat\right\|_{2,\mu}^2 \leq \frac{16 V_{\max}^2 \ln \frac{2 |\Qcal|}{\delta}}{3n} + 2 \varepsilon_2 + \varepsilon_3 + \sqrt{\frac{8 V_{\max}^2 \ln \frac{2 |\Qcal|}{\delta}}{n} \left( \frac{10 V_{\max}^2 \ln \frac{2 |\Qcal|}{\delta}}{3n} + 2 \varepsilon_2 + \varepsilon_3 \right)}, \\
&~ \qquad \text{where,} \quad \varepsilon_2 = \frac{43 V_{\max}^2 \ln \frac{8|\Qcal||\mathcal F|}{\delta}}{n} + \sqrt{\frac{239 V_{\max}^2 \ln \frac{8|\Qcal||\mathcal F|}{\delta}}{n} \epQF} + \epQF,\\
\label{eq:bound_minimax_blm}
&~ \qquad \text{and,} \quad \varepsilon_3 = \epQ + \sqrt{\frac{8 V_{\max}^2 \ln \frac{2 |\mathcal F|}{\delta}}{n} \epQ} + \frac{4 V_{\max}^2 \ln \frac{2|\Qcal|}{\delta}}{3n}.
\end{align}

By substitute Eq.\eqref{eq:bound_J_blm} into Eq.\eqref{eq:bound_minimax_blm} and adapt the the proof of Theorem 17 in \citet{chen2019information}, we have
\begin{align}
\max_{\pi \in \Pi_{\Qcal}} J(\pi) - J(\pi_{\Qhat}) \leq &~ \frac{2 \sqrt{C_{\text{eff}}}}{1 - \gamma} \left\|\Qhat - \mathcal T \Qhat\right\|_{2,\mu} \\
\leq &~ \frac{2 \sqrt{C_{\text{eff}}}}{1 - \gamma} \left( \sqrt{2 \epQ} + \sqrt{2 \epQF} \right) + \frac{2 \sqrt{C_{\text{eff}}}}{1 - \gamma} \left( \sqrt{\frac{24 V_{\max}^2 \ln \frac{2|\Qcal|}{\delta}}{n}} + \sqrt{\frac{172 V_{\max}^2 \ln \frac{8|\Qcal| |\mathcal F|}{\delta}}{n}} \right) \\
&~ + \frac{2 \sqrt{C_{\text{eff}}}}{1 - \gamma} \left( \sqrt[4]{\frac{32 V_{\max}^2 \ln \frac{2|\Qcal|}{\delta}}{n}\epQ} + \sqrt[4]{\frac{3824 V_{\max}^2 \ln \frac{8|\Qcal| |\mathcal F|}{\delta}}{n}\epQF} \right).  \tag*{\qedhere}
\end{align}
\end{proof}

\begin{reptheorem}{thm:batch_olive_bound}[Error bound of MABO, restated]
Let $\Qhat$ be the output of \algavg. W.p.~$1 - \delta$, 
\begin{align}
\max_{\pi \in \Pi_{\Qcal}} J(\pi) - J(\pi_{\Qhat}) \leq \frac{2}{1 - \gamma}\left(\epsQ + \epsW + \epsstat \right).
\end{align}
where
\begin{align}
\epsQ \coloneqq &~ \min_{Q \in \Qcal} \max_{w \in \Wcal}\left|\E_{\mu}[w\cdot(\Tcal Q - Q)]\right|, \\
\epsW \coloneqq &~ \max_{\pi \in \Pi_{\Qcal}} \inf_{\substack{w \in \lspan(\Wcal)}} \max_{Q \in \Qcal} \bigg|\E_{\mu} \big[(\wpi{\pi} - w) \cdot(\Tcal Q - Q)\big]\bigg|, \\ 
\epsstat \coloneqq &~ 2 \textstyle V_{\max} \sqrt{\frac{2 \CeffW \ln\frac{2|\Qcal||\Wcal|}{\delta}}{n}} + \frac{4 \CinfW V_{\max} \ln\frac{2|\Qcal||\Wcal|}{\delta}}{3n},\\
\CeffW \coloneqq &~ \max_{w\in\Wcal} \|w\|_{2, \mu}^2,  \quad
\CinfW \coloneqq \max_{w\in\Wcal} \|w\|_\infty,
\end{align}
and 
$\lspan(\Wcal)$ is the linear span of $\Wcal$ using coefficients with (at most) unit $\ell_1$ norm, i.e., 
$$ \textstyle 
\lspan(\Wcal) \coloneqq \left\{\sum_{w \in \Wcal} \alpha(w) w : \sum_{w\in\Wcal}|\alpha(w)| \le 1\right\}.$$ 
\end{reptheorem}

\begin{proof}[Proof of Theorem \ref{thm:batch_olive_bound}]
Let $\pistar \coloneqq \argmax_{\pi \in \Pi_{\Qcal}} J(\pi)$. By Theorem \ref{thm:telepd}, we have
\begin{align}
\max_{\pi \in \Pi_{\Qcal}} J(\pi) - J(\pi_{\Qhat}) \leq &~ \frac{\E_{d_{\pistar}}\left[\Tcal \Qhat - \Qhat \right]}{1 - \gamma} + \frac{\E_{d_{\pi_{\Qhat}}}\left[\Qhat - \Tcal \Qhat \right]}{1 - \gamma}\\
\leq &~ \frac{2\max_{\pi \in \Pi_{\Qcal}} \left|\mathcal L_{\mu} (\Qhat,w_{\nicefrac{d_\pi}{\mu}})\right|}{1 - \gamma}.
\end{align}
We now bound $\left|\mathcal L_{\mu} (\Qhat,w_{\nicefrac{d_\pi}{\mu}})\right|$ for any policy $\pi \in \PiQ$. 
%
Let 
\begin{align}
\widehat w_{\nicefrac{d_\pi}{\mu}} \coloneqq \argmin_{w \in \lspan(\Wcal)}\max_{Q \in \Qcal}\left|\E_{\mu} \left[ \left(w_{\nicefrac{d_\pi}{\mu}} - w\right) \cdot \left(\Tcal \Qhat - \Qhat \right)\right]\right|,
\end{align}
and we obtain
\begin{align}
\left|\mathcal L_{\mu} (\Qhat,w_{\nicefrac{d_\pi}{\mu}})\right| = &~ \left|\E_{\mu} \left[ \left(w_{\nicefrac{d_\pi}{\mu}} - \widehat w_{\nicefrac{d_\pi}{\mu}}\right) \cdot \left(\Tcal \Qhat - \Qhat \right)\right] + \E_{\mu} \left[ \widehat w_{\nicefrac{d_\pi}{\mu}} \cdot \left(\Tcal \Qhat - \Qhat \right)\right]\right| \\
\leq &~ \left|\E_{\mu} \left[ \left(w_{\nicefrac{d_\pi}{\mu}} - \widehat w_{\nicefrac{d_\pi}{\mu}}\right) \cdot \left(\Tcal \Qhat - \Qhat \right)\right]\right|  + \left|\E_{\mu} \left[ \widehat w_{\nicefrac{d_\pi}{\mu}} \cdot \left(\Tcal \Qhat - \Qhat \right)\right]\right| \\
= &~ \epsW + \left|\E_{\mu} \left[ \widehat w_{\nicefrac{d_\pi}{\mu}} \cdot \left(\Tcal \Qhat - \Qhat \right)\right]\right|,
\end{align}
where the last equation follows from the definition of $\epsW$.

To bound the remaining term, we first need a helper lemma that $\sup_{w\in\lspan(\Wcal)} |f(\cdot)| = \max_{w\in\Wcal} |f(\cdot)|$ for any linear function $f(\cdot)$: consider any $w \in \lspan(\Wcal)$, which can be written as $w = \sum_i \alpha_i w_i$, where $w_i \in \Wcal, \forall i$ and $\sum_i |\alpha_i| \leq 1$. For linear $f(\cdot)$ and any $w \in \lspan(\Wcal)$ we have
\begin{align}
\label{eq:mabo_span_bound}
\left| f(w) \right| = &~ \left| f\left(\sum_i \alpha_i w_i\right) \right| = \left| \sum_i \alpha_i f(w_i) \right| \leq \sum_i |\alpha_i| \left| f(w_i) \right| \leq \sup_{w' \in \Wcal} |f(w')|.
\end{align}
So $\sup_{w\in\lspan(\Wcal)} |f(\cdot)| \le \max_{w\in\Wcal} |f(\cdot)|$. On the other hand, since $\Wcal \subset \lspan(\Wcal)$, we conclude that $\sup_{w\in\lspan(\Wcal)} |f(\cdot)| = \max_{w\in\Wcal} |f(\cdot)|$ for linear $f(\cdot)$.

With this preparation, now we are ready to bound $\left|\E_{\mu} \left[ \widehat w_{\nicefrac{d_\pi}{\mu}} \cdot \left(\Tcal \Qhat - \Qhat\right)\right]\right|$. Note that 
$$
\epsQ \coloneqq \min_{Q \in \Qcal} \max_{w \in \Wcal}\left|\E_{\mu}[w\cdot(\Tcal Q - Q)]\right| = \min_{Q \in \Qcal} \sup_{w \in \lspan(\Wcal)}\left|\E_{\mu}[w\cdot(\Tcal Q - Q)]\right|,
$$
so we have
\begin{align}
\left|\E_{\mu} \left[ \widehat w_{\nicefrac{d_\pi}{\mu}} \cdot \left(\Tcal \Qhat - \Qhat \right)\right]\right| 
= \left|\E_{\mu} \left[ \widehat w_{\nicefrac{d_\pi}{\mu}} \cdot \left(\Tcal \Qhat - \Qhat \right)\right]\right| - \min_{Q \in \Qcal}\sup_{w \in \lspan(\Wcal)}\left|\E_{\mu} \left[ w \cdot \left(\Tcal Q - Q \right)\right]\right| + \epsQ,
\end{align}
At this point, we peeled off all the approximation errors from $\left|\mathcal L_{\mu} (\Qhat,w_{\nicefrac{d_\pi}{\mu}})\right|$, and it remains to bound the estimation error
\begin{align}
\left|\E_{\mu} \left[ \widehat w_{\nicefrac{d_\pi}{\mu}} \cdot \left(\Tcal \Qhat - \Qhat \right)\right]\right| - \inf_{Q \in \Qcal}\sup_{w \in \lspan(\Wcal)}\left|\E_{\mu} \left[ w \cdot \left(\Tcal Q - Q \right)\right]\right|.
\end{align}
Let $\Qtilde \coloneqq \argmin_{Q \in \Qcal}\sup_{w \in \lspan(\Wcal)}\left|\E_{\mu} \left[ w \cdot \left(\Tcal Q - Q \right)\right]\right|$ and $\Wcal_{1} \coloneqq \{a w: a \in [-1, 1], w \in \Wcal \}$.
\begin{align}
&~ \left|\E_{\mu} \left[ \widehat w_{\nicefrac{d_\pi}{\mu}} \cdot \left(\Tcal \Qhat - \Qhat \right)\right]\right| - \inf_{Q \in \Qcal}\sup_{w \in \lspan(\Wcal)}\left|\E_{\mu} \left[ w \cdot \left(\Tcal Q - Q \right)\right]\right| \\
\leq &~ \sup_{w \in \lspan(\Wcal)} \left|\E_{\mu} \left[ w \cdot \left(\Tcal \Qhat - \Qhat \right)\right]\right| - \sup_{w \in \lspan(\Wcal)}\left|\E_{\mu} \left[ w \cdot \left(\Tcal \Qtilde - \Qtilde \right)\right]\right| \\
= &~ \sup_{w \in \lspan(\Wcal)} \left|\E_{\mu} \left[ w \cdot \left(\Tcal \Qhat - \Qhat \right)\right]\right| - \sup_{w \in \lspan(\Wcal)}\left|\E_{(s,a,r,s') \sim \mathcal D} \left[ w(s,a) \left(r + \max_{a'} \Qhat(s',a') - \Qhat(s,a) \right)\right]\right| \\
&~ + \sup_{w \in \lspan(\Wcal)} \left|\E_{(s,a,r,s') \sim \mathcal D} \left[ w(s,a) \left(r + \max_{a'} \Qhat(s',a') - \Qhat(s,a) \right)\right]\right| - \sup_{w \in \lspan(\Wcal)}\left|\E_{\mu} \left[ w \cdot \left(\Tcal \Qtilde - \Qtilde \right)\right]\right| \\
\overset{\text{(a)}}{\leq} &~ \sup_{w \in \lspan(\Wcal)} \left|\E_{\mu} \left[ w \cdot \left(\Tcal \Qhat - \Qhat \right)\right] - \E_{(s,a,r,s') \sim \mathcal D} \left[ w(s,a) \left(r + \max_{a'} \Qhat(s',a') - \Qhat(s,a) \right)\right]\right| \\
&~ + \sup_{w \in \lspan(\Wcal)} \left|\E_{(s,a,r,s') \sim \mathcal D} \left[ w(s,a) \left(r + \max_{a'} \Qhat(s',a') - \Qhat(s,a) \right)\right]\right| - \sup_{w \in \lspan(\Wcal)}\left|\E_{\mu} \left[ w \cdot \left(\Tcal \Qtilde - \Qtilde \right)\right]\right| \\
\label{eq:mabo_2terms}
\overset{\text{(b)}}{\leq} &~ \underbrace{\sup_{w \in \Wcal} \left|\E_{\mu} \left[ w \cdot \left(\Tcal \Qhat - \Qhat \right)\right] - \E_{(s,a,r,s') \sim \mathcal D} \left[ w(s,a) \left(r + \max_{a'} \Qhat(s',a') - \Qhat(s,a) \right)\right]\right|}_{\text{(I)}} \\
&~ + \underbrace{\sup_{w \in \Wcal} \left|\E_{(s,a,r,s') \sim \mathcal D} \left[ w(s,a) \left(r + \max_{a'} \Qtilde(s',a') - \Qtilde(s,a) \right)\right] - \E_{\mu} \left[ w(s,a) \left(\Tcal \Qtilde - \Qtilde \right)\right]\right|}_{\text{(II)}}.
\end{align}
where (a) follows form $\sup_x|f(x)| - \sup_x|g(x)| \leq \sup_x |f(x) - g(x)|$ and (b) follows from Eq.\eqref{eq:mabo_span_bound} and the following argument:
\begin{align}
&~ \sup_{w \in \lspan(\Wcal)} \left|\E_{(s,a,r,s') \sim \mathcal D} \left[ w(s,a) \left(r + \max_{a'} \Qhat(s',a') - \Qhat(s,a) \right)\right]\right| - \sup_{w \in \lspan(\Wcal)}\left|\E_{\mu} \left[ w \cdot \left(\Tcal \Qtilde - \Qtilde \right)\right]\right| \\
\leq &~ \sup_{w \in \Wcal} \left|\E_{(s,a,r,s') \sim \mathcal D} \left[ w(s,a) \left(r + \max_{a'} \Qhat(s',a') - \Qhat(s,a) \right)\right]\right| - \sup_{w \in \Wcal_{1}}\left|\E_{\mu} \left[ w \cdot \left(\Tcal \Qtilde - \Qtilde \right)\right]\right| \\
\leq &~ \sup_{w \in \Wcal} \left|\E_{(s,a,r,s') \sim \mathcal D} \left[ w(s,a) \left(r + \max_{a'} \Qhat(s',a') - \Qhat(s,a) \right)\right]\right| - \sup_{w \in \Wcal}\left|\E_{\mu} \left[ w \cdot \left(\Tcal \Qtilde - \Qtilde \right)\right]\right| \\
\leq &~ \sup_{w \in \Wcal} \left|\E_{(s,a,r,s') \sim \mathcal D} \left[ w(s,a) \left(r + \max_{a'} \Qtilde(s',a') - \Qtilde(s,a) \right)\right]\right| - \sup_{w \in \Wcal}\left|\E_{\mu} \left[ w \cdot \left(\Tcal \Qtilde - \Qtilde \right)\right]\right| \\
\leq &~ \sup_{w \in \Wcal} \left|\E_{(s,a,r,s') \sim \mathcal D} \left[ w(s,a) \left(r + \max_{a'} \Qtilde(s',a') - \Qtilde(s,a) \right)\right] - \E_{\mu} \left[ w \cdot \left(\Tcal \Qtilde - \Qtilde \right)\right]\right|,
\end{align}
where the first inequality follows from Eq.\eqref{eq:mabo_span_bound} and the fact that $\Wcal_1 \subseteq \lspan(\Wcal)$, the second inequality follows from the linearity of $\E_{\mu} \left[ w \cdot \left(\Tcal \Qtilde - \Qtilde \right)\right]$, the third inequality follows from the fact that $\Qhat$ optimizes $\max_{w\in\Wcal} |\Lcal_\Dcal(\cdot, w)|$, and the last inequality follows from $\sup_x|f(x)| - \sup_x|g(x)| \leq \sup_x |f(x) - g(x)|$.

Now, since the only difference between term (I) and term (II) is the choice of $Q$ and $w$, it suffices to provide a uniform deviation bound that applies to all $w\in\Wcal$ and $Q\in\Qcal$. Before applying concentration bounds, it will be useful to first verify the boundedness of the random variables: $w(s,a) \in [-C, C]$, and $r + \gamma \max_{a'} Q(s', a') - Q(s,a) \in [-\Vmax, \Vmax]$ (recall that we assumed $Q\in[0, \Vmax]$). Therefore, by Bernstein's inequality and the union bound, w.p.~at least $1 - \delta$ we have that for any $w\in\Wcal$ and $Q\in\Qcal$,
\begin{align}
&~ \left| \E_\mu \left[w \cdot \left(\Tcal Q - Q \right) \right] - \frac{1}{n}\sum_{i = 1}^n \left[w(s_i,a_i) \left(r_i + \gamma \max_{a'}Q(s_i',a') - Q(s_i,a_i) \right) \right] \right| \\
\leq &~ \sqrt{\frac{2 \Var_{\mu} \left[ w(s,a) \left(r + \gamma \max_{a'}Q(s',a') - Q(s,a) \right) \right] \ln\frac{2|\Qcal||\Wcal|}{\delta}}{n}} + \frac{2 \CinfW V_{\max} \ln\frac{2|\Qcal||\Wcal|}{\delta}}{3n} \\
\label{eq:mabo_stat_n}
\overset{\text{(a)}}{\leq} &~ V_{\max} \sqrt{\frac{2 \CeffW \ln\frac{2|\Qcal||\Wcal|}{\delta}}{n}} + \frac{2 \CinfW V_{\max} \ln\frac{2|\Qcal||\Wcal|}{\delta}}{3n} = \frac{\epsstat}{2},
\end{align}
where (a) is obtained by the following argument:
\begin{align}
&~ \Var_{\mu} \left[ w(s,a) \left(r + \gamma \max_{a'}Q(s',a') - Q(s,a) \right) \right] \\
\leq &~ \E_{\mu} \left[ w(s,a)^2 \left(r + \gamma \max_{a'}Q(s',a') - Q(s,a) \right)^2 \right] \\
\leq &~ V_{\max}^2 \E_{\mu} \left[ w(s,a)^2 \right] 
\leq V_{\max}^2 \CeffW. 
\end{align}

Substituting Eq.\eqref{eq:mabo_stat_n} into Eq.\eqref{eq:mabo_2terms}, we obtain that the both of term (I) and term (II) in Eq.\eqref{eq:mabo_2terms} can be simultaneously bounded by $\epsstat/2$ w.p.~$1 - \delta$ . 
Therefore, we bound $\max_{\pi \in \Pi_{\Qcal}} J(\pi) - J(\pi_{\Qhat})$ w.p.~$1 - \delta$ as follows
\begin{align}
\max_{\pi \in \Pi_{\Qcal}} J(\pi) - J(\pi_{\Qhat}) \leq \frac{2}{1 - \gamma}\left( \epsQ + \epsW + \epsstat \right). \tag*{\qedhere}
\end{align}
\end{proof}

%% file: sections/appendix_arxiv.tex
\section{Comparison between Per-step vs.~Occupancy-based Concentrability Coefficients} \label{sec:per_step_con}

We provide an example to illustrate the limitation of the per-step concentrability coefficients (Proposition~\ref{prop:per_step_con}). Consider a deterministic chain MDP, where there are $L + 1$ states, $\{s_0, s_1, s_2, \dotsc, s_L\}$. There is only one action, which we omit in the notations. $s_0$ is the deterministic initial state, and each $s_l$  transitions to $s_{l+1}$ under the only action for $0\le l <L$.  $s_L$ is an absorbing state (i.e., it transitions to itself). The reward function is inconsequential.

There is only one possible policy $\pi$ for this MDP, and we let the data distribution $\mu = d_\pi$. The occupancy-based concentrability coefficient is always $1$ (either $C_\infty$ or $C_{\text{eff}}$), which agrees with the intuition that there is no distribution shift. Since the per-step definitions (Eq.\eqref{eq:Cps}) are always the convex combinations of $C_t = \max_{\pi} \|\wpi{\pi,t}\|_\infty$ for  $t \ge 0$, we can assert that it is never lower than $\min_t C_t$ however the combination coefficients are chosen. 

Now we calculate $C_t$ for this MDP:
\begin{align*}
C_t = \begin{cases}
\frac{1}{\mu(s_t)} = \frac{1}{(1 - \gamma) \gamma^t}, ~~ 0 \le t < L \\
\frac{1}{\mu(s_L)} = \frac{1}{\gamma^L}, ~~ t \ge L
\end{cases}
\end{align*}
Replacing $\|\cdot\|_\infty$ with $\|\cdot\|_{2,\mu}^2$ gives exactly the same results. (When the distribution on the enumerator is a point mass, $\|\cdot\|_{2,\mu}^2$ of the importance weight is the same as $\|\cdot\|_\infty$.) Therefore, as long as $L$ is sufficiently large so that $\frac{1}{\gamma^L} \ge \frac{1}{(1 - \gamma)}$, we have $C_t \ge \nicefrac{1}{1 - \gamma}$ for all $t$, and the per-step concentrability coefficient is at least $\nicefrac{1}{1 - \gamma}$. 
As a final remark, since the MDP only has 1 policy, the result has no dependence on the choice of policy class in $\max_\pi$ in the definition of concentrability coefficient, so we have virtually covered all existing definitions in the AVI/API literature. 

\section{On Iterative Methods' Lack of Control of Bellman Errors} \label{sec:iter_lack_berr}

We demonstrate that iterative methods fail to directly control  the Bellman error on the data distribution $\mu$. Consider a two-state deterministic MDP with just 1 action, where $s_1$ transitions to $s_2$, and $s_2$ is absorbing. The reward is always $0$. 
%

We use the tabular representation for this MDP, where $Q = [Q(s_1,a), Q(s_2,a)]^\top$. Assume our batch data $\mathcal D$ only contains transition tuples of form $(s_1,a,0,s_2)$, and no data points from $(s_2, a_2)$ are present. We first show how FQI behave on this example. Given the update rule of FQI (Eq.\eqref{eq:fqi_update_rule}), 
\begin{align}
Q_t \in \argmin_{Q} \ell_{\mathcal D}(Q;Q_{t - 1}) = \{[Q(s_1,a), Q(s_2,a)]^\top: Q(s_1,a) = \gamma Q_{t - 1}(s_2,a)\}.
\end{align}
Therefore, with very update, $Q(s_1,a)$ will obtain the old value of $\gamma Q(s_2,a)$ from the previous iteration, whereas the new value of $Q(s_2,a)$ will be set arbitrarily. Since the mean square Bellman error is $\|Q_t - \mathcal T Q_t\|_{2,\mu}^2 = (Q_t(s_1,a) - \gamma Q_t(s_2,a))^2$, its value can be arbitrarily away from $0$ and do not become smaller over iterations. In comparison, it is easy to verify that \algsq and \algavg do not suffer from this issue: although there is also arbitrariness in their outputs due to insufficient data coverage, their outputs will always satisfy $Q(s_1,a) = \gamma Q(s_2,a)$ and hence imply zero Bellman error on $\mu$.

As a final remark, it should be noted that the counterexample holds because $\mu$ is non-exploratory and $\Ceff = \Cinf = \infty$, which breaks the assumption for all algorithms considered in this paper. Although $\|Q- \Tcal Q\|_{\mu,2}^2$ will be controlled by FQI when $\mu$ is exploratory, this is an indirect consequence of FQI finding $Q\approx Q^\star$, and our example illustrates that these iterative methods do not \emph{directly} control the Bellman error on the data distribution. 

\section[Existence of Simple W in Low-rank MDPs (Proposition~\ref{prop:lowrk})]{Existence of Simple $\Wcal$ in Low-rank MDPs (Proposition~\ref{prop:lowrk})}
\label{sec:lowrk}

\paragraph{Claim 1: General Low-rank Case} 
Consider an MDP whose  transition matrix $P \in \RR^{|\Scal\times\Acal|\times|\Scal|}$ satisfies $\text{rank}(P) = k$. 
Let there be a total of $N$ policies in $\Pi_{\Qcal}$, and we stack $\nu_\pi \in \RR^{|\Scal|}$ for all $\pi \in \Pi_{\Qcal}$ as a matrix: 
$
M_{\nu} \coloneqq 
\begin{bmatrix}
\nu_{\pi_1} & 
\cdots &
\nu_{\pi_N}
\end{bmatrix}^\top
$; all vectors in this proof are treated as column vectors. 
We first argue that $\text{rank}(M_\nu) \le k+1$.

Let $\nu_{\pi,t}(s)$ be the marginal distribution of $s_t$ under $\pi$. Also let $\Pi \in \RR^{|\Scal| \times |\Scal\Acal|}$ be the standard matrix representation of a policy $\pi$, that is, $\Pi_{s', (s,a)} \coloneqq \mathds 1(s=s', a = \pi(s))$. It is known that $\nu_{\pi,t}^\top = d_0^\top (\Pi P)^t$, which shows that $\nu_{\pi,t}^\top$ is in the row-space of $\left[\begin{smallmatrix} P \\ d_0^\top \end{smallmatrix}\right]$ for any $\pi$ and $t$. Since $\nu_\pi = (1-\gamma) \sum_{t=0}^\infty \gamma^t \nu_{\pi,t}$, the same holds for $\nu_\pi$. Therefore, we have $\text{rank}(M_{\nu}) \leq \text{rank}(\left[\begin{smallmatrix} P \\ d_0^\top \end{smallmatrix}\right]) \le k+1$. For convenience, let $k' \coloneqq k+1$.

Then, following a determinant(volume)-maximization argument similar to~\citet[Proposition 10]{chen2019information}, we can find $k'$ rows from $M_{\nu}$, denoted as $\eta_1, \dotsc, \eta_{k'}$, which satisfies the following: for any $i = 1, \dotsc, N$, there exists $\alpha_1, \dotsc, \alpha_{k'}$, such that $\nu_{\pi_i} = \sum_{j = 1}^{k'} \alpha_j \cdot k' \cdot \eta_j$, and $|\alpha_j| \leq \nicefrac{1}{k'}$ for $j = 1, \dotsc, k'$. This implies that $\{\nu_{\pi_1}, \dotsc, \nu_{\pi_N}\} \subseteq \lspan(\{\eta_1', \dotsc, \eta_{k'}'\})$, where $\eta_i' \coloneqq k'\eta_i$. Now consider $\lspan(\{\eta_1', \dotsc, \eta_{k'}'\} \times \PiQ$, where the Cartesian product  produces $k'|\PiQ|$ pairs of state-action functions, defined as $\eta' \times \pi \coloneqq ( (s,a) \mapsto \eta'(s) \cdot \mathds{1}(a=\pi(s)))$. We claim that $\{d_{\pi_1}, \ldots, d_{\pi_N}\} \subset \lspan(\{\eta_1', \dotsc, \eta_{k'}'\} \times \PiQ$: for any $\pi_i$, since $\nu_{\pi_i}$ can be expressed as the linear combination of $\{\eta_1', \ldots, \eta_{k'}'\}$ with coefficients satisfying the norm constraints, $d_{\pi_i} = \nu_{\pi_i} \times \pi_i$ is also the combination of $\{\eta_1' \times \pi_i, \ldots, \eta_{k'}' \times \pi_i\}$ with exactly the same coefficients.

Since $\mu$ is supported on the entire $\Scal\times\Acal$, we have $\wpi{\pi} = \text{diag}(\mu)^{-1} d_\pi$. Putting all results together, it suffices to choose $\Wcal = \{\text{diag}(\mu)^{-1} (\eta_i' \times \pi_Q) : i \in [k'], Q \in \Qcal\}$, and $|\Wcal| \le (k+1)|\PiQ|$.


\paragraph{Remark on the $|\Qcal|$ Dependence in the General Case}
The annoying dependence on $|\Qcal|$ comes from the fact that we hope the \emph{state-action occupancy} vectors of different policies to have low-rank factorization (which is satisfied in the more restricted case; see Claim 2). In general low-rank MDPs, however, only state occupancy factorizes and the state-action one does not; a counter-example can be easily shown in contextual bandits:

Consider an MDP with 2 actions per state. $d_0$ is uniform among $|\Scal|-1$ states, all of which transition deterministically to the last state, which is absorbing. This MDP essentially emulates a contextual bandit. Since all states share exactly the same next-state distribution, the rank of the transition matrix is $1$ regardless of how large $|\Scal|$ is. 
Now consider a policy space $\PiQ$, where each policy takes action $a_1$ in one of the $|\Scal|-1$ states, and takes $a_2$ in all other states; there are $|\Scal|-1$ such policies. It is easy to show that the matrix consisting of state-action occupancy $d_\pi$ for all policies in $\PiQ$ has full-rank $|\Scal|-1$, which cannot be bounded by the rank of the transition matrix when $|\Scal|$ is large. 

Given this difficulty, our strategy is to first find the policies whose state occupancy vectors span the entire low-dimensional space, and take their Cartesian product with $\PiQ$ to handle the actions, which results in the $|\Qcal|$ dependence. As we will see below, we can avoid paying $|\Qcal|$ when the $\Qcal$ class is more structured.

\paragraph{Claim 2: Restricted Case of Knowing the Left Factorization Matrix as Features~\citep{yang2019sample}}
Here we consider the setting of $P = \Phi P'$, where $\Phi \in \RR^{|\Scal\times\Acal|\times k}$ and $\phi(s,a)^\top$ denotes its $(s,a)$-th row. For the choice of $\Qcal = \{(s,a) \mapsto R(s,a) + \gamma \phi(s,a)^\top \theta: \theta \in \RR^k\}$, note that any $Q\in \Qcal$ is in the column space of $\Phi^+ \coloneqq [\Phi~ R]$, where the reward function $R$ is treated as an $|\Scal\times\Acal|\times 1$ vector. \citet[Proposition 2]{yang2019sample} shows that it is realizable and closed under Bellman update, i.e., $\Tcal Q \in \Qcal, \forall Q\in\Qcal$. Therefore, the Bellman error $Q - \Tcal Q$ is also in the column space of $\Phi^+$. Let $\phi^+(s,a)^\top$ be the $(s,a)$-th row of $\Phi^+$, and $\theta_Q^+$ and $\theta_{\Tcal Q}^+$ be the coefficients such that $Q = \phi^+(s,a)^\top \theta_Q^+$ and $\Tcal Q = \phi^+(s,a)^\top \theta_{\Tcal Q}^+$.

Fixing any $\pi$, consider
\begin{align}
&~ \E_{\mu}[(w-\wpi{\pi})\cdot(\Tcal Q - Q)] \\
= &~ \E_{\mu}[(w-\wpi{\pi})\cdot (\phi^+(s,a)^\top (\theta_Q^+ -\theta_{\Tcal Q}^+))]\\
= &~ (w-\wpi{\pi})^\top \text{diag}(\mu) \Phi^+ (\theta_Q^+ -\theta_{\Tcal Q}^+).
\end{align}
According to the definition of $\epsW$, to achieve $\epsW=0$ it suffices to have the following: for every $\pi \in \PiQ$, there exists $w\in \lspan(\Wcal)$, such that $\E_{\mu}[(w-\wpi{\pi})\cdot (Q-\Tcal Q)] = 0$ for any $Q\in\Qcal$. Given the linear structure of $Q$ and $\Tcal Q$, we can relax the last statement to its sufficient condition: 
$$
(w - \wpi{\pi})^\top \text{diag}(\mu) \Phi^+ = \mathbf{0}_{k+1}^\top,
$$
where $\mathbf{0}$ is the all-zero vector. 
The rest of the proof is very similar to Claim 1: we simply stack $\wpi{\pi}^\top \text{diag}(\mu) \Phi^+ \in \RR^{1\times (k+1)}$ together into a $|\PiQ|\times(k+1)$ matrix, use the determinant-maximization argument to select its rows, and form $\Wcal$ with the corresponding $\wpi{\pi}$ after proper rescaling.

\paragraph{Remark} Since $\Qcal$ is closed under Bellman update in this setting, one may also use $\Qcal$ as the helper class $\Fcal$ for \algsq. However, the complexity of $\Fcal$ in this case only matches that of $\Wcal$ in the more general case (Claim 1) and is significant worse than what we can achieve here ($|\Wcal| \le k+1$). 